%% file: kernelFunctionFI.tex
\documentclass[12pt,twoside,a4paper]{article}
\usepackage[OT4]{fontenc}
\usepackage[cp1250]{inputenc}
\usepackage[dvips]{color}
\usepackage{amsfonts}
\usepackage{amsmath}
\usepackage{graphicx}
\usepackage{hyperref}
\usepackage{float}
\usepackage{url}
\usepackage{array}
\usepackage{algorithm,algpseudocode}

\newcommand{\Bem}[1]{}

\newcommand{\eref}[1]{(\ref{#1})}
\newcommand{\T}{^T}

\newcommand{\K}{K}
\newcommand{\kk}{\kappa}

\newcommand{\FI}{\left(\mathbf{I}-\mathbf{1}\mathbf{s}^T\right)}
\newcommand{\Fi}[1]{\left(\mathbf{I}-\mathbf{1}\mathbf{#1}^T\right)}
\newcommand{\FIT}{\left(\mathbf{I}-\mathbf{s}\mathbf{1}^T\right)}
\newcommand{\FiT}[1]{\left(\mathbf{I}-\mathbf{#1}\mathbf{1}^T\right)}

\newcommand{\KMEANSSTUFF}[1]{#1}
\newcommand{\KS}[1]{#1}

\newtheorem{theorem}{Theorem}
\newtheorem{proof}{Proof}

\begin{document}
\newcommand{\MoovjTytulv}{Validity of Clusters Produced By kernel-$k$-means With  Kernel-Trick }
\newcommand{\MaInstytucja}{Institute of Computer Science \\of the Polish Academy of Sciences\\ul. Jana Kazimierza 5, 01-248 Warszawa
Poland,\\ \url{ klopotek@ipipan.waw.pl  }}
\title{
\MoovjTytulv  }
\author{
Mieczys{\l}aw A. K{\l}opotek  
\MaInstytucja
}
 
\maketitle

\begin{abstract}
This paper, constituting an extension to the conference paper \cite{MAK:2017},  corrects the proof of the Theorem 2 from the 
Gower's paper \cite[page 5]{Gower:1982}
\KS{
as well as corrects the Theorem 7 from Gower's paper \cite{Gower:1986}
}. 
The \KS{first} correction is needed in order to establish the existence of  the kernel function used commonly in the kernel trick e.g. for $k$-means clustering algorithm, on the grounds of distance matrix. 
The correction encompasses the missing  if-part proof and dropping unnecessary conditions.
\KS{
The second correction deals with transformation of the kernel matrix 
into a one embeddable in Euclidean space.
}
\end{abstract}
\noindent


\section{The  Problem}

A number of approaches to solving various data mining problems, including clustering, is based on so-called kernel approach. 
The kernel approach may be seen as application of a mapping $\Phi$ to the data
points in such a way that they are represented in a   high dimensional Euclidean   space (called feature space) in which it is hoped to separate the data points easier via simpler  geometrical constructs (e.g. hyperplanes), compared for example to their original low dimensional representation space. 
In this way, a number of data mining methods requiring linear data separation can be applied to non-linearly separated data sets.   

The kernel approach is most frequently applied in conjunction with Support Vector Machine based analysis methods, but it is also used in case of $k$-means clustering algorithm\footnote{For an overview of kernel $k$-means algorithm see e.g. \cite{Dhillon:2004}.}, in which we are interested in this paper. 
We will introduced this algorithm in Section 
\ref{sec:kernelKmeans}

The kernel-based approaches assume the availability of a similarity function $\kk()$ and in particular of the similarity matrix $K$, called also a kernel function and kernel  matrix resp., which express  similarities between data points at hand. This similarity function/matrix must have the property that, for any  two data points  $\mathbf{i}   , \mathbf{j}$ in the original apace space  we have $\kk(  \mathbf{i}   , \mathbf{j})=\Phi(\mathbf{i}) \circ \Phi( 
\mathbf{j})$ ($\circ$ operator indicates a dot product between vectors), and for any two data points in the data set under consideration the similarity matrix $K$ is available such that $ K_{ij}=\kk(  \mathbf{i}   , \mathbf{j})$. 

 For a number of algorithms, including $k$-means, the so-called kernel trick has been elaborated. 
The essence of the kernel trick is that
we can perform the kernel algorithms based on 
 the kernel matrix $K$ alone, without an explicit knowledge of the mapping $\Phi$. 
Section \ref{sec:kernelKmeans} explains the usage of kernel trick for $k$-means algorithm. 

Nonetheless, the very existence of the mapping $\Phi$, and hence of the kernel function $\kk()$ is of vital importance to the validity of application of the $k$-means algorithm in the feature space. $\Phi$ transforms the data to points in an Euclidean space so that $k$-means can be applied at all. Inversion of $\Phi$ will provide with cluster centers produced by kernel-$k$-means. Furthermore, not similarities but rather distances are used by $k$-means. 
We can easily imagine that no kernel function $\kk()$ exists for a given similarity matrix $K$. We can also have to do with the situation that there exist multiple kernel functions $\kk$ as well as $\Phi$ related to the same kernel matrix $K$. Can it mean that there exist multiple feature spaces in which the very same data set can be clustered differently via kernel-$k$-means depending on the $\Phi$ function we choose?
 Closely related is the following issue: 
For algorithms like $k$-means,   instead of the kernel matrix   the distance matrix $D$ between the objects in the feature space may be available, being  the Euclidean distance matrix. We will call $D$ Euclidean matrix.

We are faced with the following questions:
\begin{itemize}
\item[(1)] what properties the kernel matrix should have in  be really a matrix of dot products?  
\item[(2)] what properties the kernel matrix should have in   to enable to recover function  $\Phi()$ at the data points from the kernel matrix? 
\item[(3)]   can we obtain the matrix $K$ from distance data matrix $D$?
\item[(4)]   can we obtain from the matrix $K$ the function $\Phi()$ 
such that the distances in the feature space are exactly the same as given by the $D$ matrix?  
\item[(5)]  if we derived the matrix $K$ from $D$ and $K$ turns out to yield $\Phi()$, can we know then that $D$ was really an Euclidean distance matrix?
\end{itemize}

Questions (1), (2) may seem to be pretty easy and were partially addressed 
e.g. by  Sch\"{o}lkopf \cite{Schoelkopf:2001}.
Sch\"{o}lkopf investigates what kinds of kernel functions 
may lead to a distance measure in the feature space.
However, he does not consider the inverse, that is Euclidean distance matrix leading to a kernel function. 
He does not investigate finding explicit form of  the $\Phi$ function either. 

The answer to the third question seems to be easily derivable from the paper   by Balaji et al. 
\cite{Balaji:2007}. One should use   the transformation 
\begin{equation} \label{eq:BalajiTransform}
 K=-\frac12\left(\mathbf{I}-\frac{\mathbf{1}\mathbf{1}^T}{m}\right)D_{sq} 
(\mathbf{I}-\frac{\mathbf{1}\mathbf{1}^T}{m})
\end{equation}
(where $D_{sq}$ is a matrix containing as entries squared distances from $D$)
a result going back to a paper by Schoenberg \cite{Schoenberg:1938}. 
The problem is that this paper of Schoenberg  does not contain any such statement. 
This result should be rather ascribed to the paper  
\cite{Schoenberg:1935}.
\footnote{
Schoenberg \cite{Schoenberg:1938} proposed still another 
distance-to-kernel matrix transform 
$$\kk_d(\mathbf{x},\mathbf{y})=exp(-\gamma d^2(\mathbf{x},\mathbf{y}))$$
for any positive $\gamma$, which we will not discuss here. 
}
The most general proposal of a distance-to-kernel-matrix transform seems to be that of  by Gower \cite[Theorem 2, page 5]{Gower:1982},  who   generalizes the aforementioned  transform \eref{eq:BalajiTransform}  to 
\begin{equation}\label{eq:GowerTransform}
K= \FI (-\frac12 D_{sq}) \FIT 
\end{equation}
for an appropriate choice of $\mathbf{s}$. 
A  generally accepted proof of this transformation can be found in the  paper by 
Gower  \cite[Theorem 2, page 5]{Gower:1982}.
If this proof were correct, the questions (4) and (5) would have been answered. 
Regrettably, the proof of the validity of the latter is  incomplete, as we will explain in   Section \ref{sec:gowerformulation}. 
For this reason, these questions still remain open. 

Therefore, we decided to provide with a correction of the proof of the Gower's theorem that we will present in 
Section \ref{sec:correction}.
  This correction is needed in order to establish the existence of 
the kernel function used commonly in the kernel trick e.g. for $k$-means clustering algorithm, on the grounds of distance matrix.

The question that was left open by Gower was: do there exist special cases where two different    $\Phi()$ functions, complying with a given kernel matrix, generate different distance matrices in the feature space, maybe in some special, "sublimated" cases? 
 This would mean that under some "special" conditions the output of kernel $k$-means could differ radically not just on the grounds of some random causes but in a systematic way.  
The answer given to this open question  in this paper is definitely NO. 
We closed all the conceivable gaps in this respect. 
So usage of (linear and non-linear) kernel matrices that are semipositive definite, is safe in this respect.

Let us underline here that we did not impose any apriorical restrictions on the form of $\Phi()$ function itself. 
It may be a linear or non-linear mapping from the sample space to the feature space. But what we insist on is that the feature space has to be Euclidean. 
This is the requirement for applicability of (kernel) $k$-means  clustering algorithm. 
If the feature space is not metric, the results of  (kernel)  $k$-means clustering are questionable. 

 In Section \ref{sec:example} we provide with a numerical example illustrating some distance matrix  transformations discussed in Section \ref{sec:correction}.

\KMEANSSTUFF{
The second problem with usage of kernel-$k$-means is related to the basic assumption of $k$-means that it has been designed for Euclidean space. 
In a number of applications, like clustering based on Laplacians, 
the embeddability of the kernel matrix can be guaranteed from the theoretical standpoint.
However, this does not need to be always the case. 
Therefore we need to answer the questions (6) what does kernel-$k$-means produce for non-Euclidean kernel matrices, (7) can a non-Euclidean kernel matrix be turned to an  Euclidean kernel matrix, (8) how does the latter matrix transformation impact the results of kernel-$k$-means clustering. 
The questions could have been easily answered if the 
Theorem 7 of Gower from \cite{Gower:1986} were correct. Regrettably, this Theorem requires an quantitative correction. 
We handle these issues in Section \ref{sec:nonEuclideanKernels}. 
}

In the subsequent Section \ref{sec:background} we will point at research directions for which the correction proposed here is of importance.

\section{The Background}\label{sec:background}

The $k$-means algorithm has the very attractive property of being easy to implement, and there exist various variants of it like $k$-means++ possessing even closeness-to-optimum properties. 
The drawback of this algorithm is that it accepts numeric attributes only and requires an embedding in Euclidean space. Embeddings into other spaces were investigated, like hyperbolic space, but the computation of cluster centers that is vital and very easy in Euclidean space, is not that easy in the other spaces. 

However, real-world objects are frequently described by non-numeric attributes, or are not embedded in any space whatsoever and instead only similarity, dissimilarity or distance between objects is known. In such cases the kernel-$k$-means clustering algorithm can be used which at least partially inherits the good properties of $k$-means. In such cases, however, the very existence of embedding into Euclidean space (even if it is not used explicitly), is of vital importance, because otherwise the clustering results may be unreliable. Same holds for other kernel  algorithms for which the original algorithm relies on an Euclidean space. 

Therefore, research is performed like that of \cite{Li:2013}, in order to find ways of transforming a similarity matrix into the closest  proper positive definite kernel matrix, so that an approximating Euclidean embedding is existent, or one learns the distances themselves. 
 
These efforts in establishing the proper kernel matrix make sense only if the Theorem 2 of Gower \cite{Gower:1982} is valid.  
However, a  study of the literature seems to reveal that nobody except for Gower himself was aware of the mentioned flaw of his proof of his theorem and the result is used rather as a granted truth. 

The Gower's paper \cite{Gower:1982}, according to GoogleScholar, is cited over 200 times in a number of research and application contexts.
For example, 
Pekalska et al. \cite{Pevkalska:2002}  derive the necessity of creation of a generalized kernel handling of dissimilarity on the grounds that the
 kernel according to equation \eref{eq:GowerTransform} is positive definite if and only if the underlying
 distance matrix 
is Euclidean, which has not been proven by Gower \cite{Gower:1982}. 
Same motivation lies behind 
Nikolentzos et al. work \cite{Nikolentzos:2017} on seeking appropriate embeddings. 
Pavoine et al. \cite{Pavoine:2004} relies on the  property, suggested by Gower \cite{Gower:1982}, that the decomposition of the kernel can be shifted, while performing PCA analysis.   

Kernel-trick based $k$-means algorithms are  applied in various areas 
(e.g. gene expression clustering \cite{Handhayania:2015},
spectral clustering of graphs \cite{Dhillon:2004}).

The validity of the Gower transform underpins various improvements of kernel $k$-means   clustering, like single pass clustering \cite{Sarma:2013}. 
global kernel $k$-means \cite{Tzortzis:2009}, 
subsampling kernel $k$-means \cite{Chitta:2011}
robust kernel $k$-means \cite{Yaqiang:2018}
and other. 

\KS{
Furthermore, let us stress here  that the aforementioned papers do not care at all about whether or not the kernel matrices are embeddable in Euclidean space which is the basic assumption of applying the basic form of kernel-$k$-means. 
Non-Euclidean space require a serious modification of $k$-means, accommodating to that fact that gravity center of a cluster cannot serve any more as cluster center
(gradient descent methods are needed for example, see 
\cite[Section 6]{Richter:2017}. 
}

For these reasons a definite solving of the Gower theorem dilemma seems to be of uttermost importance.

\section{Kernel-$k$-means}\label{sec:kernelKmeans}

The well known $k$-means clustering  algorithm 
is claimed to minimize the objective function being the 
sum of squares of distances of data points to their cluster centers.
It  
consists of the following steps: (1) 
creating the initial clustering, (2) 
computation of cluster center for each cluster, 
(3) creation of a new clustering by assigning each data point 
to the cluster defined by the closest cluster center 
(4) repeating steps (2) and (3) till some terminating condition. 
There exist a large variety of variants of this algorithm.
For example step (1) may cosist in random selection of $k$ distinct data points 
as cluster centers and applying step (3). 
Another variant may replace step (2) with step (2') in which a single data point is moved from one cluster to the other if and only if the move decreases the cost function and then perform proper step (2). steps (2) and (2') may be applied interchangingly in subsequent iterations and so on.

Kernel based  $k$-means clustering algorithm (clustering objects $1,\dots,m$ into $k$ clusters $1,\dots,k$)
consists in  switching to a multidimensional feature space  $\mathcal{F}$ and searching therein for prototypes $\boldsymbol{\mu}_j^\Phi$ minimizing the error
\begin{equation}\label{CLU:eq-ker1}
   \sum_{i=1}^m \min_{1\le j \le k} \|\Phi( i) - \boldsymbol{\mu}_j^\Phi\|^2
\end{equation}
\noindent where  $\Phi\colon \{1,\dots,m\}  \to \mathcal{F}$ is a (usually non-linear) mapping of the space of objects  into the feature space.  
The so-called "kernel trick" means the possibility to apply $k$-means clustering without knowing explicitly the $\Phi(i)$ function and using so-called kernel matrix with elements $k_{ij}=\Phi( i)\T \Phi( j)=\K( i, j)$
instead.

In analogy to the classical  $k$-means algorithm, the prototype vectors are updated according to the equation
 
\begin{equation}
\boldsymbol{\mu}_j^\Phi = \frac{1}{m_j} \sum_{ i \in C_j} \Phi( i)  
\label{CLU:eq-ker2a}
\end{equation}

\noindent where $m_j  $ is the cardinality of the   $j$-{th} cluster.
  A direct application of this equation is not possible
unless   the function   $\Phi$ is   known. 
But it may be still feasible if we would know the so-called Kernel Matrix $K$ 
with elements being dot products of data points in the feature space, that is 
$k_{ij}=\Phi( i)\T \Phi( j)=\K( i, j)$.
Given matrix $K$,   it is possible to compute the distances between the object images and prototypes in the feature space by making use of so-called called "the kernel trick". The "kernel trick" 
relies on the fact that the following transformation is possible:

\begin{equation}
\begin{array}{l}
\|\Phi( i) - \boldsymbol{\mu}_j^\Phi\|^2 =
 \displaystyle \big(\Phi( i) - \boldsymbol{\mu}_j^\Phi\big)\T \big(\Phi( i) 
- \boldsymbol{\mu}_j^\Phi \big) \vspace{0.3cm}\\
\phantom{\|\Phi( i) - \boldsymbol{\mu}_j^\Phi\|^2} = \Phi( i)\T \Phi( i) - 2\Phi( i)\T \boldsymbol{\mu}_j^\Phi + (\boldsymbol{\mu}_j^\Phi)\T \boldsymbol{\mu}_j^\Phi \vspace{0.3cm}\\
\phantom{\|\Phi( i) - \boldsymbol{\mu}_j^\Phi\|^2} = \Phi( i)\T\Phi( i) - \displaystyle\frac{2}{m_j} \sum_{h \in C_j} \Phi( i)\T \Phi( h) +  \vspace{0.3cm}\\
\hspace{3.2cm} + \displaystyle\frac{1}{m_j^2}\sum_{r \in C_j} \sum_{s  \in C_j} \Phi( r)\T\Phi( s) \vspace{0.3cm}\\
\phantom{\|\Phi( i) - \boldsymbol{\mu}_j^\Phi\|^2} = 
\displaystyle k_{ii}  - \frac{2}{m_j}\sum_{h \in C_j}  k_{hi} + \frac{1}{m_j^2}\sum_{r\in C_j} \sum_{s \in C_j}^m  k_{rs}
\end{array}\label{CLU:eq-kkk}
\end{equation}

\noindent where, as already stated, $k_{ij}=\Phi( i)\T \Phi( j)=\K( i, j)$.

 In this way, one can update the elements of clusters   without determining the prototypes explicitly.

Let $Y$ be a matrix $Y=(\Phi( 1),\Phi( 2),\dots,\Phi( m))^T$. 
Then apparently $K=YY^T$. Hence for any non-zero vector $\mathbf{u}$ 
$\mathbf{u}^TK\mathbf{u}=\mathbf{u}^TYY^T\mathbf{u}=(Y^T\mathbf{u})^(Y^T\mathbf{u})=\mathbf{y}^T\mathbf{y}\ge 0$ where $\mathbf{y}=Y^T\mathbf{u}$
so $K$ must be positive semidefinite. 
But a matrix is  positive semidefinite iff all its eigenvalues are non-negative.  
Furthermore, all its eigenvectors are real numbers. 

So to identify $\Phi()$ at data points, 
one has to find  all eigenvalues $\lambda_l$, $l=1,\dots,m$
and corresponding eigenvectors  $\mathbf{v}_l$ of the matrix $K$. 
If all eigenvalues are hereby non-negative, then construct the matrix   
 $Y$ that has as columns the products $\sqrt{\lambda_l}\mathbf{v}_l$.
Rows of this matrix (up to permutations) are the values of the function  $\Phi()$ at data points $1,\dots,m$.  
More formally, if the matrix $V=(\mathbf{v}_1,\dots, \mathbf{v}_m)$, and $\Lambda$ is the vector of eigenvalues, then 
\begin{equation}\label{eq:phimatrix}
Y=V diag(\sqrt{\Lambda})
\end{equation}
where $diag()$ turns a vector into a diagonal matrix. 
It may be verified that kernel-$k$-means with the above $K$ matrix  and ordinary $k$-means for $Y$ would yield  same results.

 \section{Gower formulation of distance-to-kernel-matrix transformation}\label{sec:gowerformulation}
Let us recall that a matrix $D\in \mathbb{R}^{m\times m}$ is an Euclidean distance matrix between points $1,\dots,m$ if and only if there exists a matrix $X\in   \mathbb{R}^{m\times n}$ rows of which ($\mathbf{x_1}^T,\dots,\mathbf{x_m}^T$) are coordinate vectors of these points in an $n$-dimensional Euclidean space and
\begin{equation}
d_{ij}=\sqrt{(\mathbf{x_i}-\mathbf{x_j})^T(\mathbf{x_i}-\mathbf{x_j}) }
\end{equation}
. 
Gower in \cite{Gower:1982} claims that 
\begin{theorem}\label{th1982}
$D$ is Euclidean iff 
the matrix $F=\FI (-\frac12) D_{sq} \FIT$ is positive semidefinite 
for any vector $\mathbf{s}$ such that 
$\mathbf{s}^T\mathbf{1}=1$ and $D_{sq}\mathbf{s}  \ne \mathbf{0}$
\end{theorem}
{
whereas in  \cite{Gower:1986} he claims:
\begin{theorem}\label{th1986}
$D$ is Euclidean iff 
the matrix $F=\FI (-\frac12)D_{sq} \FIT$ is positive semidefinite 
for any vector $\mathbf{s}$ such that 
$\mathbf{s}^T\mathbf{1}=1$.
\end{theorem}

Apparently both claims do not match quite (with respect to condition $D_{sq}\mathbf{s}  \ne \mathbf{0}$).  
It must be underlined, however, that the paper \cite{Gower:1982} provides strong clues how the theorem  \ref{th1986} shall be proven, though incompletely, 
so that in what follows we use these clues to establish the result. 
}%
We claim here is that the Gower's theorem has the following deficiencies
\begin{itemize}
\item requirement $D_{sq}\mathbf{s}  \ne \mathbf{0}$ is not needed in Theorem \ref{th1982}.
\item the  if-part of {neither} Theorem    \ref{th1982}   nor of  his   theorem correction in  \cite{Gower:1986} \Bem{ in Theorem \ref{th1986} } was demonstrated. 
\end{itemize}
%
{
It should be noted at this point, that in a 1985 paper 
Gower \cite{Gower:1985} 
derives his theorem in the latter version 
from a paper by Schoenberg \cite{Schoenberg:1935}. 
The problem is that first of all Gower's result does not need this second derivation 
and second the paper by Schoenberg \cite{Schoenberg:1935} does not prove what Gower \cite{Gower:1985} claims. So the issue is open and we want to address it here more thoroughly. 
}%
We provide a coorection, completing Gower's proof in Section \ref{sec:correction}.
See Section  \ref{sec:example} for some numerical examples of matrices and vectors that we operate on in  Section \ref{sec:correction}.
In Section \ref{sec:conclusions} we draw some conclusions from the corrective proof.

\section{Correrction of Gower's result}\label{sec:correction}
In this section we shall correct the Gower's result from \cite{Gower:1982}.

For construction purposes we need still another formulation of the theorem, which is slightly more elaborate:

\begin{theorem}\label{thEuclideanembeddingMAKversion}
\begin{enumerate}
\item If the matrix $D$ is a matrix of Euclidean distances then 
for each  vector $\mathbf{s}$ such that 
$\mathbf{s}^T\mathbf{1}=1$
the matrix 
\begin{equation}\label{eq:multiplikativeformtransform}
 F=\FI (-\frac12)D_{sq} \FIT
\end{equation}
 is positive semidefinite ($D_{sq}$ being a matrix with entries being squares of entries of the matrix $D$). 
\item If 
$D$ is a symmetric matrix with zero diagonal and 
for a  vector $\mathbf{s}$ such that 
$\mathbf{s}^T\mathbf{1}=1$.
the matrix $F=\FI (-\frac12)D_{sq} \FIT$ is positive semidefinite 
then  $D$ is Euclidean. 
\item If $D$ is Euclidean then 
for each  vector $\mathbf{s}$ such that 
$\mathbf{s}^T\mathbf{1}=1$
the matrix $D$ can be derived from 
matrix $F=\FI (-\frac12)D_{sq} \FIT$ in such a way that its squared entries can be computed 
as $d_{ij}^2=f_{ii}+f_{jj}-2f_{ij}$.
\item If $D$ is Euclidean then 
for each  vector $\mathbf{s}$ such that 
$\mathbf{s}^T\mathbf{1}=1$
the 
matrix $F=\FI (-\frac12)D_{sq} \FIT$ 
can be expressed as $F=YY^T$ where $Y$ is a real-valued matrix, 
and the rows of $Y$ can be considered as coordinates of data points the distances between which are those from the matrix $D$. 
\end{enumerate}
\end{theorem}

Let $D\in \mathbb{R}^{m\times m}$ be a matrix of Euclidean distances $d_{ij}$ between objects $i,j\in \{1,\dots,m\}$. 
Let $D_{sq}$ be a matrix of squared Euclidean distances $d_{ij}^2$ between objects with identifiers $1,\dots,m$.   
This means that there must exist a matrix $X\in \mathbb{R}^{m\times n}$ for some $n$,  rows  of which represent coordinates of these objects in an $n$-dimensional space. This real-valued matrix $X$ represents an embedding of the Euclidean distance matrix $D$ into $\mathbb{R}^{m\times n}$.   A distance matrix can be called Euclidean if and only if an embedding exists. 
If $E=XX^T$ ($E$ with dimensions $m\times m$), then $d^2_{ij}=e_{ii}+e_{jj}-2e_{ij}$.  

As a rigid set of points in Euclidean space can be moved (shifted, rotated, flipped symmetrically\footnote{Gower does not consider flipping.})  without changing their relative distances,  there may exist many other matrices $Y$ 
 rows of which represent coordinates of these same objects in the same  $n$-dimensional space after some isomorphic transformation.  
Let us denote the set of all such embeddings $\mathcal{E}(D)$. 
And if a matrix $Y\in \mathcal{E}(D)$, then for the product  $F=YY^T$ 
we have   $d^2_{ij}=f_{ii}+f_{jj}-2f_{ij}$.
We will say that $F\in \mathcal{E}_{dp}(D)$

For an  $F\in \mathcal{E}_{dp}(D)$ 
define a matrix 
 $G=F+\frac12D_{sq}$. 
Hence  $F=G-\frac12D_{sq}$. 
Obviously then
\begin{align}
 d^2_{ij}& =f_{ii}+f_{jj}-2f_{ij}\\
&=
(g_{ii}- \frac12d^2_{ii})+(g_{jj}- \frac12d^2_{jj})-2(g_{ij}-\frac12 d^2_{ij})
\\ & =
  g_{ii} +g_{jj} -2 g_{ij}+d^2_{ij}
\end{align}
(as $d_{jj}=0$ for all $j$). 
This implies that 
\begin{equation}
0=  g_{ii} +g_{jj} -2 g_{ij}
\end{equation}
\noindent
   that is  
\begin{equation}
 g_{ij}=\frac{g_{ii} +g_{jj}}{2}
\end{equation}
 So $G$ is of the form 
 \begin{equation}
 G=\mathbf{g}\mathbf{1}^T+ \mathbf{1}\mathbf{g}^T 
 \end{equation}
with components of $\mathbf{g}\in \mathbb{R}^m$ equal $g_i=\frac12g_{ii}$.

Therefore, to find  $F\in \mathcal{E}_{dp}(D)$
for an Euclidean matrix $D$ 
  we need only to consider matrices deviating from $-\frac12D_{sq}$ by $\mathbf{g}\mathbf{1}^T+ \mathbf{1}\mathbf{g}^T$ for some $\mathbf{g}$. 
Let us denote with $\mathcal{G}(D)$ the set of all matrices $F$ such that 
$F=\mathbf{g}\mathbf{1}^T+ \mathbf{1}\mathbf{g}^T-\frac12D_{sq}$. 
So for each matrix $F$ if
$F\in \mathcal{E}_{dp}(D)$ then   $F\in \mathcal{G}(D)$, but not vice versa.
We stress that we work with an Euclidean matrix $D$. 
So we would like to find an $F$ such that $F$ is decomposable into real-valued matrices $Y$ such that $F=YY^T$ so that $Y$ would represent an embedding of an Euclidean distance matrix. 
But first of all even if $D$ is not Euclidean, or even not metric, such an embedding may be found. 
(see  Gower et al. \cite{Gower:1986}). 
 
As Gower et al. \cite{Gower:1986} states,
see their Theorem 1, any non-metric dissimilarity measure $ d(\mathfrak{z}, \mathfrak{y})$ 
for $\mathfrak{z}, \mathfrak{y}\in  {\mathfrak{X}}$ where $ {\mathfrak{X}}$ is finite, can be turned into a (metric) distance function 
$ d'(\mathfrak{z}, \mathfrak{y})= d(\mathfrak{z}, \mathfrak{y})+c $ 
where $c$ is a constant where $c\ge \max_{\mathfrak{x}, \mathfrak{y}, \mathfrak{z}\in \mathfrak{X}} \|d(\mathfrak{x}, \mathfrak{y}) +d(\mathfrak{y}, \mathfrak{z})-d(\mathfrak{z}, \mathfrak{x})  \|$. 
Furthermore, 
Gower et al.  \cite{Gower:1986} recall that any dissimilarity matrix $D$ may be turned to an Euclidean distance matrix, see their Theorem 7, by adding an appropriate constant, e.g.  
 $ d'(\mathfrak{z}, \mathfrak{y})=
 \sqrt{
 d(\mathfrak{z}, \mathfrak{y})^2+\sigma }$ 
 where $\sigma$ is a constant such that
 $\sigma\ge -\lambda_m$, $\lambda_m$ being the smallest eigenvalue of 
 $(\mathbf{I}-\mathbf{1}\mathbf{1}^T/m) (-\frac12  D_{sq}) (\mathbf{I}-\mathbf{1}\mathbf{1}^T/m)$, $D_{sq}$ is the matrix of squared values of elements of $D$,  $m$ is the number of rows/columns in $D$. 
 
 So
even if $D$ is actually an Euclidean distance matrix, 
and
$F=  -\frac12D_{sq}+\mathbf{g}\mathbf{1}^T+ \mathbf{1}\mathbf{g}^T$, 
there is no warranty, that the distance matrix induced by corresponding $Y$ is identical with $D$.

For an $F\in \mathcal{G}(D)$ consider the matrix $ F^*=\FI F \FI^T$. 
We obtain 
\begin{align}
 F^*=&\FI F \FI^T
\\ =&	\FI (\mathbf{1}\mathbf{g}^T  + \mathbf{g}\mathbf{1}^T-\frac12D_{sq}) \FI^T 
\\ =&	\FI \mathbf{1}\mathbf{g}^T\FI^T  + \FI\mathbf{g}\mathbf{1}^T\FI^T
\nonumber \\ & 	-\frac12\FI D_{sq}\FI^T 
\label{eq:FIFFIT}
\end{align}

Let us investigate $\FI \mathbf{1}\mathbf{g}^T\FI^T$:
\begin{equation}
\FI \mathbf{1}\mathbf{g}^T\FIT=
\mathbf{1}\mathbf{g}^T -\mathbf{1}\mathbf{g}^T \mathbf{s}\mathbf{1}^T 
-  \mathbf{1}\mathbf{s}^T \mathbf{1}\mathbf{g}^T+ \mathbf{1}\mathbf{s}^T \mathbf{1}\mathbf{g}^T \mathbf{s}\mathbf{1}^T
\end{equation}

Let us make the following choice (always possible) of 
$\mathbf{s}$ with respect to $\mathbf{g}$:
$\mathbf{s}^T\mathbf{1}=1$, $\mathbf{s}^T \mathbf{g}=0$.

Then we obtain from the above equation
\begin{equation}
\FI \mathbf{1}\mathbf{g}^T\FIT=
\mathbf{1}\mathbf{g}^T -\mathbf{1} 0 \mathbf{1}^T 
-   \mathbf{1}\mathbf{g}^T+ \mathbf{1}\mathbf{s}^T \mathbf{1}\cdot 0 \cdot \mathbf{1}^T
=\mathbf{0}\mathbf{0}^T
\label{eq:FI1gFIT}
\end{equation}
By analogy  
\begin{equation}
\FI\mathbf{g}\mathbf{1}^T\FI^T
= (\FI \mathbf{1}\mathbf{g}^T\FIT)^T=\mathbf{0}\mathbf{0}^T
\label{eq:FIg1FIT}
\end{equation}

By substituting (\ref{eq:FI1gFIT}) and (\ref{eq:FIg1FIT}) into 
(\ref{eq:FIFFIT})
we obtain 
\begin{align}
 F^*&=\FI F \FI^T= -\frac12\FI D_{sq}\FI^T 
\label{eq:M12FIDFIT}
\end{align}

So for any $\mathbf{g}$, hence an $F\in \mathcal{G}(D)$ we can find an $\mathbf{s}$ such that:
$$\FI F \FI^T= -\frac12\FI D_{sq}\FI^T$$  

For any matrix $F=-\frac12\FI D_{sq}\FI^T$ for some $\mathbf{s}$ with $\mathbf{1}^T\mathbf{s}=1$ 
we say that $F$ is in multiplicative form or $F\in \mathcal{M}(D)$.

If   $F=YY^T$, that is $F$ is decomposable,
then also $$F^*=\FI YY^T \FI^T=(\FI Y)  (\FI Y)^T={Y^*}{Y^*}^T$$ is decomposable. 
But 
\begin{equation}
Y^*=\FI Y= Y-\mathbf{1}\mathbf{s}^T Y = Y-\mathbf{1}\mathbf{v}^T  
\end{equation}
where $\mathbf{v}= Y^T \mathbf{s}$ is a shift vector by which the 
whole matrix $Y$ is shifted to a new location in the Euclidean space. 
So the distances between objects computed from $Y^*$ are the same as those from $Y$,
hence
if $F\in \mathcal{E}_{dp}(D)$, then 
 $Y^*\in \mathcal{E}(D)$. 

Therefore, to find a matrix 
    $F\in \mathcal{E}_{dp}(D)$, yielding an embedding of $D$ 
in the Euclidean $n$ dimensional space 
we need only to consider matrices 
of the form 
$-\frac12\FI D_{sq}\FI^T$, subject to the already stated constraint $\mathbf{s}^T\mathbf{1}=1$, that is ones from $\mathcal{M}(D)$. 

So we can conclude:
If $D$ is a matrix of Euclidean distances, then 
there must exist a positive semidefinite
matrix $F=-\frac12   \FI  D_{sq} \FIT $
for some vector $\mathbf{s}$ such that 
$\mathbf{s}^T\mathbf{1}=1$,
$\det( \FI)= 0$ and 
$D_{sq}\mathbf{s}\ne \mathbf{0}$. 
These last two conditions are implied 
by the following fact:
  $D_{sq}$ is known to be not negative semidefinite, so that $F$ would not be 
positive semidefinite 
in at least the following cases:
$\det( \FI)\ne 0$ (see reasoning prior to formula (\ref{eq:det-ne-zero})) or $D_{sq}\mathbf{s}=\mathbf{0}$ (see reasoning prior to formula (\ref{eq:Ds-eq-zero})). 
So if $D$ is an Euclidean distance matrix, then 
there exists an $F \in \mathcal{M}(D) \cap \mathcal{E}_{dp}(D)$. 

Let us investigate other vectors  
 $\mathbf{t}$ such that 
$\mathbf{t}^T\mathbf{1}=1$.
Note that 
\begin{align}
(\mathbf{I}-\mathbf{1}\mathbf{t}^T)(\mathbf{I}-\mathbf{1}\mathbf{s}^T)
&=
\mathbf{I} -\mathbf{1}\mathbf{t}^T-\mathbf{1}\mathbf{s}^T
+\mathbf{1}\mathbf{t}^T \mathbf{1}\mathbf{s}^T
\\ &=
\mathbf{I} -\mathbf{1}\mathbf{t}^T-\mathbf{1}\mathbf{s}^T
+\mathbf{1} \mathbf{s}^T
=
\mathbf{I} -\mathbf{1}\mathbf{t}^T \label{eq:combine_ts_to_t}
\end{align}

Therefore, for a matrix $F\in \mathcal{M}(D)$  
\begin{align}
(\mathbf{I}-\mathbf{1}\mathbf{t}^T)F(\mathbf{I}-\mathbf{1}\mathbf{t}^T)^T
&=
-\frac12 (\mathbf{I}-\mathbf{1}\mathbf{t}^T)   \FI  D_{sq} \FI^T (\mathbf{I}-\mathbf{1}\mathbf{t}^T)^T
\nonumber \\ &= 
-\frac12 (\mathbf{I}-\mathbf{1}\mathbf{t}^T)    D_{sq}   (\mathbf{I}-\mathbf{1}\mathbf{t}^T)^T
\end{align}

But if $F=YY^T  \in \mathcal{E}_{dp}(D)$, then   

\begin{align} 
F'&=(\mathbf{I}-\mathbf{1}\mathbf{t}^T)F(\mathbf{I}-\mathbf{1}\mathbf{t}^T)^T 
\\ &=
(\mathbf{I}-\mathbf{1}\mathbf{t}^T)YY^T(\mathbf{I}-\mathbf{1}\mathbf{t}^T)^T 
\\ &=
(Y-\mathbf{1}(\mathbf{t}^T)Y) (Y-\mathbf{1}(\mathbf{t}^T)Y)^T 
\end{align} 
and hence each 
$-\frac12 (\mathbf{I}-\mathbf{1}\mathbf{t}^T)    D_{sq}   (\mathbf{I}-\mathbf{1}\mathbf{t}^T)^T$
is also in $\mathcal{E}_{dp}(D)$, though with a different placement (by a shift) in the coordinate systems of the embedded data points. 
So if one element of $\mathcal{M}(D)$ is in $\mathcal{E}_{dp}(D)$, then all of them are. 
 
So we have established that:
 if $D$ is an Euclidean distance matrix\footnote{
This means that there exists a  matrix $X$ such that rows are coordinates of objects in an Euclidean space with distances as in $D$}, then   there exists a decomposable matrix $F=YY^T\in \mathcal{E}_{dp}(D)$ which is in $\mathcal{G}(D)$, hence $ \mathcal{E}_{dp}(D) \subset  \mathcal{G}(D) $.  
For each matrix in $\mathcal{G}(D) \cap  \mathcal{E}_{dp}(D)$ 
 there exists a multiplicative form matrix 
in $\mathcal{M}(D) \cap  \mathcal{E}_{dp}(D)$.  
But if it exists, all multiplicative forms are there: $\mathcal{M}(D) \subset  \mathcal{E}_{dp}(D)$

In this way we have proven points 1,3 and 4 of the Theorem \ref{thEuclideanembeddingMAKversion}.
 And also the only-if-part of 
Gower's theorem correction in  \cite{Gower:1986}.

However, two things remain to be clarified and are not addressed in \cite{Gower:1982} nor in  \cite{Gower:1986}:
 the if-part of  \cite{Gower:1986} theorem correction 
(given a matrix $D$ such that 
$-0.5 \FI D_{sq} \FI^T$ is positive semidefinite,  is $D$ an Euclidean distance matrix? -- see point 2 of the Theorem \ref{thEuclideanembeddingMAKversion})
and the status  of the additional condition 
$D_{sq}\mathbf{s}  \ne \mathbf{0}$
 in Theorem \ref{th1982}.

Gower \cite{Gower:1982} makes the following remark: 
  $F=\FI (-\frac12 D_{sq}) \FIT $ is to be  positive semidefinite for Euclidean $D$. 
However, for non-zero vectors $\mathbf{u}$
 
\begin{align}
 \mathbf{u}^T F \mathbf{u} =&  -\frac12 \mathbf{u}^T \FI D_{sq}\FI^T \mathbf{u}
\nonumber \\=&   -\frac12 ( \FI^T \mathbf{u})^T  D_{sq}(\FI^T \mathbf{u})
\end{align}
But $D_{sq}$ is known to be not negative semidefinite, so that $F$ would not be 
positive semidefinite 
in at least the following cases:
$\det( \FI)\ne 0$ and $D_{sq}\mathbf{s}=\mathbf{0}$. 
Let  us have a brief look at these conditions and why they are neither welcome nor actually existent:
 
\begin{enumerate}
\item Situation $\det( \FI)\ne 0$ 
{\it is not welcome}, because    
there   exists a  vector $\mathbf{u'}$ such  
that $\mathbf{u'}^TD_{sq}\mathbf{u'}>0$ 
and under $\det( \FI)\ne 0$
we could solve the equation  
$\FI^T \mathbf{u}= \mathbf{u'}$ and thus demonstrate that for some $\mathbf{u}$
\begin{equation}\label{eq:det-ne-zero}
 \mathbf{u}^T F \mathbf{u} <0 
\end{equation}
However this situation {\it is impossible}, because for $F\in \mathcal{M}(D)$  
$$(\mathbf{I}-\mathbf{1}\mathbf{s}^T)\mathbf{1}=\mathbf{1}-\mathbf{1}=\mathbf{0}$$ 
which means that the rows are linearly dependent, hence 
$\det( \FI)= 0$ is guaranteed by earlier assumption about 
$\mathbf{s}$; so this concern by Gower needs to be dismissed as pointless.  
\item  Situation  $D_{sq}\mathbf{s}=\mathbf{0}$ {\it is not welcome},  because then 
$$\mathbf{u}^T \FI D_{sq}\FI^T \mathbf{u}=
\mathbf{u}^T    D_{sq}\FI^T \mathbf{u}=
\mathbf{u}^T     \mathbf{u}>0$$ and thus 
\begin{equation}\label{eq:Ds-eq-zero}
 \mathbf{u}^T F \mathbf{u} <0 
\end{equation}
  denying positive semidefiniteness of $F$. 
Gower does not consider this further, but 
such a situation {\it is impossible}. 
Recall that because $D$  is Euclidean, there must exist a vector $\mathbf{r}$ 
such that $\mathbf{r}^T\mathbf{1}=1$ 
and 
$$F^{(r)}=YY^T=-\frac12\Fi{r}D_{sq}  \FiT{r}$$ is in 
$\mathcal{E}_{dp}(D)$.
Hence for any  $\mathbf{s}$ 
such that $\mathbf{s}^T\mathbf{1}=1$ 
\begin{align}
 \FI F^{(r)} \FIT =&(\FI Y)(\FI Y)^T
 \nonumber \\=& -\frac12\Fi{s}D_{sq}  \FiT{s}
\end{align}
 is 
positive semidefinite. 
This allows us to conclude that for such $\mathbf{s}$ 
$D\mathbf{s}\ne  \mathbf{0}   $.
Therefore if 
$D\mathbf{s}=  \mathbf{0}   $ then $\mathbf{s}^T\mathbf{1}=0$ .
What is more,   if $det(D)\ne 0$ then $D_{sq}\mathbf{s}=\mathbf{0}$ implies $\mathbf{s}=\mathbf{0}$, 
for which of course   $\mathbf{s}^1\mathbf{1}=0$.
\\ Hence the last assumption of  if-part of Theorem \ref{th1982} needs to be dropped as unnecessary
 which simplifies it to corrected theorem in \cite{Gower:1986}.

\Bem{
Assume then that $det(D)= 0$. So there must exist a submatrix $D'$ of $D$ such that $det(D')\ne 0$ and another submatrix one row and column bigger $D''$ containing $D'$ such that  $det(D'')= 0$ where diagonals of both consist of diagonal elements of $D$.  
$F''= \FI (-\frac{D''_{sq}) \FIT $ will face exactly the same problem if we choose $\mathbf{s}$ such that  $D''_{sq}\mathbf{s}=\mathbf{0}$.
Note that if $D$ is Euclidean, so are $D'$ and $D''$. 
Let $s''$ be the element of $\mathbf{s}$ corresponding to the row/column of $D''$ being an extension of $D'$ and  $\mathbf{s'}$ be the remaining part of  $\mathbf{s}$. 
If $s''$ is set to zero, then $D''_{sq}\mathbf{s}=\mathbf{0}$ implies
$D'_{sq}\mathbf{s'}=\mathbf{0}$ which is a contradiction. 
So $s''\ne0$ and without restricting the generality of the discourse let $s''=1$.
This determines uniquely the vector  $\mathbf{s'}$
and requires $\mathbf{s'}^TD'_{sq}\mathbf{s'}=0 $. 
 But in this case additionally $ \mathbf{s'}^T \mathbf{1}=0$ is required.
}   }
\end{enumerate}
 
As we can see from the first point above, $F$, given by 
$$F=-0.5 \FI D_{sq} \FI^T$$ 
does not need to identify uniquely a matrix $D$, as $\FI$ is not invertible. 
Though of course it identifies  {\it  an} Euclidean distance matrix. 

Let us now demonstrate the missing part of Gower's proof that $D$ is uniquely defined given a decomposable $F$. 

So assume that for some $D$ (of which we do not know if it is Euclidean, but is symmetric and with zero diagonal),  
$F=-0.5 \FI D_{sq} \FI^T$ and $F$ is decomposable that is $F=YY^T$. 
Let $\mathcal{D}(Y)$ be the distance matrix derived from $Y$ (that is the distance matrix for which $Y$ is an embedding).
That means   $F$ is decomposable into properly distanced points
with respect to $\mathcal{D}(Y)$. 
And $F$ is  in additive form with respect to it, that is $F \in \mathcal{G}(\mathcal{D}(Y))$
Therefore there must exist some $\mathbf{s'}$ such that 
the $F'=-0.5 \Fi{s'} \mathcal{D}(Y)_{sq}  \FiT{s'}$ as valid multiplicative form with respect to $\mathcal{D}(Y)$,
and it holds that $F'= \Fi{s'} F  \FiT{s'}$. 
But recall that 
\begin{align}
 \Fi{s'} F  \FiT{s'}
  =&  \Fi{s'} (-0.5  \FI D_{sq} \FIT ) \FiT{s'}
\nonumber\\ =& -0.5 ( \Fi{s'}   \FI) D_{sq}( \Fi{s'}   \FI)^T  
\nonumber\\ =& -0.5 \Fi{s'}  D_{sq}  \FiT{s'}
\end{align}
Hence    
$$-0.5 \Fi{s'}  D_{sq}  \FiT{s'} = -0.5 \Fi{s'} \mathcal{D}(Y)_{sq}  \FiT{s'}$$ 

So we need to demonstrate that for two 
 symmetric matrices with zero diagonals $D,D'$ such that
 $-\frac12   \FI  D_{sq} \FIT = -\frac12   \FI  D'_{sq} \FIT $ 
the equation $D=D'$ holds. 

It is easy to see that
 $-\frac12   \FI  (D_{sq}-D'_{sq}) \FIT = \mathbf{0}  \mathbf{0}^T  $. 
Denote $\Delta = D_{sq}-D'_{sq}$.
$$   \FI  \Delta  \FIT = \mathbf{0}  \mathbf{0}^T  $$
$$  \Delta- \mathbf{1}\mathbf{s}^T \Delta 
- \Delta  \mathbf{s}\mathbf{1}^T 
+\mathbf{1}\mathbf{s}^T \Delta \mathbf{s}\mathbf{1}^T
     = \mathbf{0}  \mathbf{0}^T  $$
With $\boldsymbol{\overline{\Delta}}$ denote the vector  $\Delta \mathbf{s}$ and with $c$ the scaler
$\mathbf{s}^T \Delta \mathbf{s}$.
So we have  
$$  \Delta- \mathbf{1}\boldsymbol{\overline{\Delta}}^T 
- \boldsymbol{\overline{\Delta}}\mathbf{1}^T 
+c\mathbf{1} \mathbf{1}^T
  = \mathbf{0}  \mathbf{0}^T  $$
So in the row $i$, column $j$ of the above equation we have:
$\delta_{ij}+c-\overline{\delta}_i- \overline{\delta}_j =0$.
Let us add cells $ii$ and $jj$ and subtract from them cells $ij$ and $ji$. 
$ \delta_{ii}+c-\overline{\delta}_i- \overline{\delta}_i
+ \delta_{jj}+c-\overline{\delta}_j- \overline{\delta}_j
- \delta_{ij}-c+\overline{\delta}_i+ \overline{\delta}_j
- \delta_{ji}-c+\overline{\delta}_j+ \overline{\delta}_i
= \delta_{ii}+ \delta_{jj}- \delta_{ij}- \delta_{ji}
 =0$.
But as the diagonals of $D$ and $D'$ are zeros, 
hence  $\delta_{ii}= \delta_{jj}=0$.
So  $- \delta_{ij}- \delta_{ji}
 =0$. But   $\delta_{ij}= \delta_{ji}$ because $D,D'$ are symmetric. 
Hence $-2\delta_{ji}=0$ so $\delta_{ji}=0$. 
This means that $D=D'$.

This means that $D$ and $\mathcal{D}(Y)$ are identical.
Hence    decomposition of $F=-0.5 \FI D_{sq} \FI^T$  is sufficient to prove Euclidean space embedding of $D$ and yields this embedding.
This proves the  if-part of Gower's Theorem \ref{th1982} and of the corrected theorem in \cite{Gower:1986} {\ref{th1986}} and point 2 of Theorem \ref{thEuclideanembeddingMAKversion}. 

\Bem{
So consider now a symmetric matrix $D$ with zero diagonal 
and let 
$F=-\frac12\FI D_{sq}\FI^T$.
Note that 
$\FI F\FI^T=-\frac12\FI\FI D_{sq}\FI^T\FI^T=-\frac12\FI D_{sq}\FI^T=F$.
Let it happen that $F$ is decomposable, that is $F=YY^T$. 
So $Y$ may be considered a matrix of coordinates in Euclidean space and let $D'$ be the matrix of distances  computed directly from $Y$. 
Let 
 $F'=-\frac12\FI D'_{sq}\FI^T$.


This leads to the conclusion that the matrix $K$ shall be chosen as
$K=-0.5 \FI D_{aq} \FIT$. 
Upon decomposing  $K=YY^T$ the 
$\Phi$ function 
assigns each object the corresponding row from the $Y$ matrix. 
 }

\section{A numerical example}\label{sec:example}
\input{kernelFunction_experiments}


\section{$k$-means under non-Euclidean kernels}\label{sec:nonEuclideanKernels}

In many cases, like Laplacians of graphs, we know in advance that they can be deemed as kernels embedded into Euclidean space, so that there are no obstacles to apply kernel-$k$-means clustering. 
However, this does not always need to be the case. 
Let us discuss now the concerns for applying kernel-$k$-means in such situations and about the validity of the obtained clusters.

Let $w_1,\dots,w_m$ be non-negative weights of data points $1,\dots,m$. 
Let $C$ be such a subset of $\{1,\dots,m\}$ that 
$\sum_{ i \in C} w_i \ne 0$. Define $\boldsymbol{\mu}_{\mathbf{w}}^\Phi(C)$ as a weighted center of the datapoints of $C$ as follows:
\begin{equation}
\boldsymbol{\mu}_{\mathbf{w}}^\Phi(C) =  \frac{1}{\sum_{ i \in C} w_i}\sum_{ i \in C} w_i \Phi( i)  
\label{CLU:eq-ker2d}
\end{equation}

It is easily seen that it is possible to compute 
the squared distance of  any data point 
 to a weighted center of a set.  

\begin{equation}
\begin{array}{l}
\|\Phi( i) - \boldsymbol{\mu}_{\mathbf{w}}^\Phi(C)\|^2 =
 \displaystyle \big(\Phi( i) - \boldsymbol{\mu}_{\mathbf{w}}^\Phi(C)\big)\T \big(\Phi( i) 
- \boldsymbol{\mu}_{\mathbf{w}}^\Phi(C) \big) \vspace{0.3cm}\\
\phantom{\|\Phi( i) - \boldsymbol{\mu}_{\mathbf{w}}^\Phi(C)\|^2} 
= \Phi( i)\T \Phi( i) - 2\Phi( i)\T \boldsymbol{\mu}_{\mathbf{w}}^\Phi(C) + (\boldsymbol{\mu}_{\mathbf{w}}^\Phi(C))\T \boldsymbol{\mu}_{\mathbf{w}}^\Phi(C) \vspace{0.3cm}\\
\phantom{\|\Phi( i) - \boldsymbol{\mu}_{\mathbf{w}}^\Phi(C)\|^2} 
= \Phi( i)\T\Phi( i) - \displaystyle\frac{2}{\sum_{ h \in C} w_h} \sum_{h \in C } w_h \Phi( i)\T \Phi( h) +  \vspace{0.3cm}\\
\hspace{3.2cm} + \displaystyle\frac{1}{(\sum_{ h \in C} w_h)^2}\sum_{r \in C} \sum_{s  \in C}w_r w_s \Phi( r)\T\Phi( s) \vspace{0.3cm}\\
\phantom{\|\Phi( i) - \boldsymbol{\mu}_{\mathbf{w}}^\Phi(C)\|^2} 
= 
\displaystyle k_{ii}  - \frac{2}{\sum_{ h \in C} w_h}\sum_{h \in C} w_h k_{hi} + \frac{1}{(\sum_{ h \in C} w_h)^2}\sum_{r\in C_j} \sum_{s \in C_j}^m w_r w_s k_{rs}
\end{array}\label{CLU:eq-wc}
\end{equation}

Let us now pay some attention to the consequence of the fact that 
one may be tempted to apply the kernel-$k$-means algorithm under missing Euclidean embedding.

The kernel-$k$-means algorithm
consists in  switching to a multidimensional feature space  $\mathcal{F}$ and 
it is clamed to search therein for prototypes $\boldsymbol{\mu}_j^\Phi$ minimizing the error
\begin{equation*}
   \sum_{i=1}^m \min_{1\le j \le k} \|\Phi( i) - \boldsymbol{\mu}_j^\Phi\|^2
\end{equation*}
over all possible choices of 
the set of cluster centers $\boldsymbol{\mu}_j^\Phi$, $j=1,\dots,k$. 

But this is actually not the entire truth. 
 $\boldsymbol{\mu}_j^\Phi$ may only be equal to 
\begin{equation}
\boldsymbol{\mu}_j^\Phi = \frac{1}{m_j} \sum_{ i \in C_j} \Phi( i)  
\label{CLU:eq-ker2c}
\end{equation}
for some subset $C_j$ of all the data points and no other vectors in the feature space are taken into account.
If the feature space is Euclidean, it is guaranteed 
that no other vector from the feature space shall ever be considered as cluster center, because the clustering will not be optimal. 
It is not so in case of non-Euclidean feature spaces. 
To demonstrate this, we will use an example. 

\input{kernelFunction_nonEuclideanCase}
 
In this way we have proven that
\begin{theorem}
kernel-$k$-means does not optimize the cost function 
$$J(\boldsymbol{\mu}_j^\Phi;  j=1,\dots,k)=
   \sum_{i=1}^m \min_{1\le j \le k} \|\Phi( i) - \boldsymbol{\mu}_j^\Phi\|^2$$
for non-Euclidean kernel matrices. 
\end{theorem}

We have already mentioned the Gower's et al.  \cite{Gower:1986} Theorem 7, stating that 
 any dissimilarity matrix $D$ may be turned to an Euclidean distance matrix,   by adding   constant $\sigma$ to the squared distances as follows:    
 $ d'(\mathfrak{z}, \mathfrak{y})=
 \sqrt{
 d(\mathfrak{z}, \mathfrak{y})^2+\sigma }$ 
 where $\sigma$ is a constant such that
 $\sigma\ge -\lambda_m$, $\lambda_m$ being the smallest eigenvalue of 
 $(\mathbf{I}-\mathbf{1}\mathbf{1}^T/m) (-\frac12 D_{sq}) (\mathbf{I}-\mathbf{1}\mathbf{1}^T/m)$, $D_{sq}$ is the matrix of squared values of elements of $D$,  $m$ is the number of rows/columns in $D$. 

Gower's Theorem 7 is actually \emph{wrong}. 
Let us continue the above example. 
\input{kernelFunction_Gower7wrong}

Let us now propose a correction of Gower's "euclidesation" theorem:
\begin{theorem}\label{thKlvopotekEuclidesation}
Any dissimilarity matrix $D$ may be turned to an Euclidean distance matrix, see their Theorem 7, by adding an appropriate constant (to non-diagonal elements) , e.g.  
 $ d'(\mathfrak{z}, \mathfrak{y})=
 \sqrt{
 d(\mathfrak{z}, \mathfrak{y})^2+2\sigma }$ 
 where $\sigma$ is a constant such that
 $\sigma\ge -\lambda_m$, $\lambda_m$ being the smallest eigenvalue of 
 $(\mathbf{I}-\mathbf{1}\mathbf{1}^T/m) (-\frac12  D_{sq}) (\mathbf{I}-\mathbf{1}\mathbf{1}^T/m)$, $D_{sq}$ is the matrix of squared values of elements of $D$,  $m$ is the number of rows/columns in $D$. 
\end{theorem}
\begin{proof}
The equation \eref{eq:combine_ts_to_t} allows us to conclude that   given
$$ F=-\frac12\left(\mathbf{I}-\frac{\mathbf{1}\mathbf{1}^T}{m}\right)D_{sq} 
\left(\mathbf{I}-\frac{\mathbf{1}\mathbf{1}^T}{m}\right)
$$
for a dissimilarity matrix $D$, 
the following holds:
$$ F=
(\mathbf{I}-\frac{\mathbf{1}\mathbf{1}^T}{m})
F (\mathbf{I}-\frac{\mathbf{1}\mathbf{1}^T}{m}) 
=F (\mathbf{I}-\frac{\mathbf{1}\mathbf{1}^T}{m})
= (\mathbf{I}-\frac{\mathbf{1}\mathbf{1}^T}{m}) F
$$
Let $\mathbf{v}$ be  an eigenvector  of $F$ for a non-zero   eigenvalue $\lambda$.
Therefore  
$$
\lambda (\mathbf{I}-\frac{\mathbf{1}\mathbf{1}^T}{m}) \mathbf{v} 
= (\mathbf{I}-\frac{\mathbf{1}\mathbf{1}^T}{m}) F\mathbf{v} 
=F  \mathbf{v}
= F (\mathbf{I}-\frac{\mathbf{1}\mathbf{1}^T}{m}) \mathbf{v }
$$
Assuming that $\mathbf{v'}= (\mathbf{I}-\frac{\mathbf{1}\mathbf{1}^T}{m}) \mathbf{v}$, we get:
$$\lambda \mathbf{v'} = F \mathbf{v'}$$ 
which means that $\mathbf{v'}$ is also an eigenvector of $F$ for the same eigenvalue.  
Notably, The sum of components of $\mathbf{v'}$ is equal zero.

Consider now the following expression for some number $\sigma$. 

$$F'=  \left(\mathbf{I}-\frac{\mathbf{1}\mathbf{1}^T}{m}\right)\left(-\frac12 D_{sq} 
- \sigma \left(\mathbf{1}\mathbf{1}^T-\mathbf{I} \right) \right)
\left(\mathbf{I}-\frac{\mathbf{1}\mathbf{1}^T}{m}\right)
$$
$$ = 
 \left(\mathbf{I}-\frac{\mathbf{1}\mathbf{1}^T}{m}\right)
\left(-\frac12 D_{sq}  \right)
\left(\mathbf{I}-\frac{\mathbf{1}\mathbf{1}^T}{m}\right)
-
 \left(\mathbf{I}-\frac{\mathbf{1}\mathbf{1}^T}{m}\right) 
  \sigma \left(\mathbf{1}\mathbf{1}^T-\mathbf{I} \right) 
\left(\mathbf{I}-\frac{\mathbf{1}\mathbf{1}^T}{m}\right)
$$
$$ = 
 \left(\mathbf{I}-\frac{\mathbf{1}\mathbf{1}^T}{m}\right)\left(-\frac12 D_{sq}  \right)
\left(\mathbf{I}-\frac{\mathbf{1}\mathbf{1}^T}{m}\right)
+\sigma 
 \left(\mathbf{I}-\frac{\mathbf{1}\mathbf{1}^T}{m}\right) 
$$
Now consider an eigenvector $\mathbf{v'}$ of $F'$ for a non-zero eigenvalue $\lambda'$, such that the sum of its components equals zero.  For each $\lambda$ such a vector always exists.
We see immediately that 
$$F'\mathbf{v'}=
 \left(\mathbf{I}-\frac{\mathbf{1}\mathbf{1}^T}{m}\right)\left(-\frac12 D_{sq}  \right)
\left(\mathbf{I}-\frac{\mathbf{1}\mathbf{1}^T}{m}\right) \mathbf{v'}
+\sigma \mathbf{v'}
$$
$$\lambda'  \mathbf{v'}= F  \mathbf{v'}
+\sigma \mathbf{v'}
$$
$$(\lambda'-\sigma)  \mathbf{v'}= F  \mathbf{v'}
$$
that is that $(\lambda'-\sigma)$ is an eigenvalue of $F$ with eigenvector $\mathbf{v'}$.  

This means that by subtracting    $\sigma$ from
non-diagonal elements of $-\frac12 D_{sq}$ in the computation of $F$ 
we can increase its eigenvalues of eigenvectors with zero sum by $\sigma$. 
But subtracting    $\sigma$ from
non-diagonal elements of $-\frac12 D_{sq}$
means adding $\sigma$ to 
non-diagonal elements of $\frac12 D_{sq}$, or 
 adding $2\sigma$ to 
non-diagonal elements of $D_{sq}$, or 
just replacing non-diagonal elements $d_{ij}$ of $D$ with
$\sqrt{d_{ij}^2+\sigma}$. 
If we add at least the negation of the lowest eigenvalue of non-Euclidean $F$ to all its eigenvalues, then of course it turns to an Euclidean one, given that all eigenvectors with non-zero eigenvalues have zero sums of components.  

How can we now tell if all such eigenvectors have zero sums? 
In case  that all eigenvalues are different, this is simple.
As shown, each eigenvalue has the zero sum eigenvector, and this is the only one up to scaling factor.

The details of handling special cases (of identical eigenvalues) follow now.
Consider the set of all eigenvectors related to a multiple eigenvalue. 
The whole set can be represented as a linear combination of some number of orthogonal vectors from this set with the number equal to the multiplicity of the eigenvalue.
Let $\mathbf{v}$ be one of these orthogonal vectors. 
Then any linear combination of all the other orthogonal  vectors is orthogonal to $\mathbf{v}$. 
Let $\mathbf{v"}$ be an example from this combination. 
Then clearly $\mathbf{v"}^T\mathbf{v}=0$.
But also  $ \mathbf{v"}^T(F\mathbf{v})=\lambda \mathbf{v"}^T\mathbf{v}=0$.
Hence $ \mathbf{v"}^T(F(\mathbf{I}-\frac{\mathbf{1}\mathbf{1}^T}{m}) \mathbf{v})=
  \mathbf{v"}^T\lambda \mathbf{v'}=0$. 
So $\mathbf{v'}=(\mathbf{I}-\frac{\mathbf{1}\mathbf{1}^T}{m}) \mathbf{v}$  is orthogonal to $\mathbf{v"}$. 
As the latter represents any vector orthogonal to $\mathbf{v}$ of the subspace co-spanned by $\mathbf{v}$, so  $\mathbf{v'}$ must be identical to $\mathbf{v}$ up to  scaling factor. 
So the subspace of eigenvectors can be spanned by a set of orthogonal vectors with component sums equal zero. 
Therefore all the eigenvectors of $F$ have this property and hence adding the respective constant adds to all the eigenvalues of the matrix $F$. This completes the proof.
\end{proof}

Let us illustrate the Theorem\ ref{thKlvopotekEuclidesation} by continuing the previous example. 
The euclidesation of the kernel $_{nE}F$, according to Theorem \ref{thKlvopotekEuclidesation}, will lead to the following kernel matrix:

\input{kernelFunction_Gower7correction}

Note that the clustering obtained is identical with the clustering delivered by kernel-$k$-means from the original kernel matrix $_{nE}F$. 

Let us investigate this phenomenon more generally. 

\begin{theorem}
If we pursue the kernel-$k$-means clustering when seeking the optimum among cluster center sets being a subset of 
the set of  $\boldsymbol{\mu}_j^\Phi$ that may only be equal to 
\begin{equation}
\boldsymbol{\mu}_j^\Phi = \frac{1}{m_j} \sum_{ i \in C_j} \Phi( i)  
\label{CLU:eq-ker2cx}
\end{equation}
for some subset $C_j$ of all the data points and no other vectors in the feature space are taken into account,
then after adding a constant $\sigma$  to  the distance matrix as follows:
 $ d'(\mathfrak{z}, \mathfrak{y})=
 \sqrt{
 d(\mathfrak{z}, \mathfrak{y})^2+2\sigma }$ 
then the optimal clustering will remain the same. 

\end{theorem}
\begin{proof}
If we add in a cluster $C_j$ of cardinality $m_j$
for an element $i$ to all its distances  $\sigma$, then its squared distance to the cluster center 
will increase by $\sigma\frac{m_j-1}{m_j}$ because $d(i,i)=0$ is unchanged .
So in all the cluster cost function  will change by 
$\sigma\frac{m_j-1}{m_j}m_j= \sigma\cdot(m_j-1)$. 
So the overall cost function of all $k$ clusters will increase by $\sigma\cdot(m-k)$.
That is it is independent of  the actual cost function. 
Hence the optimum clustering of $k$-means, achievable by kernel-$k$-means, will remain unchanged after this addition.
\end{proof}

Under these circumstances

\begin{theorem}
For kernel-$k$-means, 
  adding a constant to squared dissimilarity measures of 
non-identical elements is 
a clustering preserving and embeddability preserving operation. 
\end{theorem} 

Note that the transformation mentioned above (1) increases all distances, 
(2) the absolute increase in distances is the largest for the smallest distances, and the smallest for the largest,
(3) therefore no new clustering structures occur under this transformation. 
We define in this way a new axiom/property of $k$-means - in that we require that clustering algorithm yields same result under  the mentioned  distance change/transformation.

The idea behind is that in the permissible
 domain for $k$-means (Euclidean) the optimum is unchanged 
if we add constant to squared distances between different elements. 
By means of conceptual extension we can carry on this assumption
 backwards into non-Euclidean distances. 

Then we need to define under what regime we compute the permissible optimum of $k$-means, because in the whole space itself it is no true. Only if we limit the permissible space in a reasonable way, we can still assume that we are computing $k$-means optimum. 
So if we agree that the kernel function $\Phi()$ for kernel $k$-means is deemed to transmit the data points into the Euclidean space under the mentioned invariance transformation, then it is permissible to apply kernel-$k$-means without checking for embeddability.


\section{Concluding remarks}\label{sec:conclusions}

 In this paper we   corrected the proof of the Theorem 2 from the 
Gower's paper \cite[page 5]{Gower:1982}. 
This correction was needed in order to establish the existence of 
the kernel function used commonly in the kernel trick e.g. for $k$-means clustering algorithm, on the grounds of distance matrix. 

Let us underline here that we did not impose any apriorical restrictions on the form of $\Phi()$ function itself. 
It may be a linear or non-linear mapping from the sample space to the feature space. But what we insist on is that the feature space has to be Euclidean. 
This is the requirement for applicability of (kernel) $k$-means  clustering algorithm. 
If the feature space is not metric, the results of  (kernel)  $k$-means clustering are questionable. 

But this is not enough. 
The same kernel matrix may be related to infinitely many $\Phi()$ functions.

The question that was left open by Gower was: do there exist special cases where two different    $\Phi()$ functions, complying with a given kernel matrix, generate different distance matrices in the feature space, maybe in some special, "sublimated" cases? 
The answer given to this open question  in this paper is definitely NO. 
We closed all the conceivable gaps in this respect. 
So usage of (linear and non-linear) kernel matrices that are semipositive definite, is safe in this respect.

\KMEANSSTUFF{  

Furthermore we resolved the issue of applicability of kernel-$k$-means for non-embeddable kernel matrices. 
If we accept the eigen-value-shift transformation as a legitimate kernel matrix transformation 
and the kernel-$k$-means clustering in the   kernel matrix obtained via such euclidesation as the valid clustering for the original kernel matrix, then we can apply kernel-$k$-means also in the non-Euclidean space.

\subsection*{Software}
Please feel free to experiment with an R package (source code)   
implementing kernel $k$-means functionality 
\url{install.packages("https://home.ipipan.waw.pl/m.klopotek/ipi_archiv/kernelKmeansAndPlusPlusDemo_1.0.tar.gz", repos = NULL, type ="source")}
}

\bibliographystyle{plain}
\bibliography{kernelFunctionFI_bib}
\end{document}

%% file: kernelFunction_experiments.tex
 Let us illustrate the process of generating a kernel matrix from a distance table and show that the distances between the objects in the feature space really match the distances of the original distance matrix.
 We took a $n=4$-dimensional  data  matrix with $m=7$ objects. 
$$ X =
\begin{pmatrix}  
  77	& 113	& 125	& 99 \\ 
  53	& 127	& 104	& 122 \\ 
  95	& 80	& 136	& 55 \\ 
  20	& 83	& 12	& 2 \\ 
  62	& 67	& 84	& 6 \\ 
  47	& 11	& 77	& 94 \\ 
  30	& 87	& 26	& 90 
\end{pmatrix}  
$$
 and  derived from it an original Euclidean distance matrix 
$$ D_0 =
\begin{pmatrix}  
  0	& 41.7	& 58.9	& 162.3	& 112.6	& 116.8	& 113 \\ 
  41.7	& 0	& 97.4	& 160.9	& 132.4	& 122.5	& 96.1 \\ 
  58.9	& 97.4	& 0	& 154.3	& 79.8	& 109.8	& 132.7 \\ 
  162.3	& 160.9	& 154.3	& 0	& 85	& 136.4	& 89.8 \\ 
  112.6	& 132.4	& 79.8	& 85	& 0	& 105.6	& 108.8 \\ 
  116.8	& 122.5	& 109.8	& 136.4	& 105.6	& 0	& 93.2 \\ 
  113	& 96.1	& 132.7	& 89.8	& 108.8	& 93.2	& 0 
\end{pmatrix}  
$$
 
 We applied to it the transformation from equation (\ref{eq:multiplikativeformtransform}) using the vector 
$$ \mathbf{s} =
\begin{pmatrix}  
  0.22	& 0.17	& 0.08	& 0.04	& 0.05	& 0.04	& 0.41 
\end{pmatrix} ^T 
$$
  and obtained the (kernel) matrix 
$$ F =
\begin{pmatrix}  
  3755	& 2570.5	& 3689.2	& -5143.7	& -615.5	& -1404	& -3110.1 \\ 
  2570.5	& 3127.9	& 367.7	& -5238.2	& -3362	& -2403.5	& -1658.6 \\ 
  3689.2	& 367.7	& 7093.4	& -2220.5	& 4207.7	& 1048.2	& -3856.9 \\ 
  -5143.7	& -5238.2	& -2220.5	& 12284.6	& 6374.8	& 376.3	& 3510.2 \\ 
  -615.5	& -3362	& 4207.7	& 6374.8	& 7685	& 1800.5	& -683.6 \\ 
  -1404	& -2403.5	& 1048.2	& 376.3	& 1800.5	& 7070	& 589.9 \\ 
  -3110.1	& -1658.6	& -3856.9	& 3510.2	& -683.6	& 589.9	& 2791.9 
\end{pmatrix}  
$$
 
 After eigen-decomposition of $F$, we get via equation (\ref{eq:phimatrix}) the embedding matrix (after ignoring columns with next to zero eigenvalues)
$$ Y =
\begin{pmatrix}  
  -50.8	& -31.6	& -12.9	& -1.5 \\ 
  -52.4	& 9.1	& -13.8	& -10.2 \\ 
  -21.1	& -81.3	& -4.1	& 5.1 \\ 
  107.6	& 0.4	& -26.1	& -4.2 \\ 
  56.4	& -65.9	& -12.5	& -2.4 \\ 
  22.3	& -22.8	& 77.7	& -4.2 \\ 
  34.9	& 38.3	& 9.1	& 5.2 
\end{pmatrix}  
$$
 which produces the distance matrix 
$$ D =
\begin{pmatrix}  
  0	& 41.7	& 58.9	& 162.3	& 112.6	& 116.8	& 113 \\ 
  41.7	& 0	& 97.4	& 160.9	& 132.4	& 122.5	& 96.1 \\ 
  58.9	& 97.4	& 0	& 154.3	& 79.8	& 109.8	& 132.7 \\ 
  162.3	& 160.9	& 154.3	& 0	& 85	& 136.4	& 89.8 \\ 
  112.6	& 132.4	& 79.8	& 85	& 0	& 105.6	& 108.8 \\ 
  116.8	& 122.5	& 109.8	& 136.4	& 105.6	& 0	& 93.2 \\ 
  113	& 96.1	& 132.7	& 89.8	& 108.8	& 93.2	& 0 
\end{pmatrix}  
$$
 
 The sum of squared differences between the corresponding entries in the distance  matrices $D$ and $D_0$ amounts to  5.727256e-25.

 It can be easily seen that $D$ is (nearly) identical with $D_0$,  though the embeddings $X$ and $Y$ differ. 
 The $k$-means algorithm as implemented in $R$ ({\tt kmeans}, centers=2,nstart=100) was run both for the embedding $X$ and $Y$ yielding the  clustering 
[ 2, 2, 2, 1, 1, 1, 1].

 A version of kernel-$k$-means, as described in this paper, was also implemented and produced for the kernel matrix $F$ the very same clustering. 
 [2, 2, 2, 1, 1, 1, 1].

 Note that same distance matrix can be turned to a kernel matrix using different $\mathbf{s}$ verctors. We applied to it the transformation from equation (\ref{eq:multiplikativeformtransform}) using the vector 
$$ \mathbf{s'} =
\begin{pmatrix}  
  0.13	& 0.09	& 0.15	& 0.12	& 0.31	& 0.08	& 0.11 
\end{pmatrix} ^T 
$$
  and obtained the (kernel) matrix 
$$ F' =
\begin{pmatrix}  
  5423.6	& 5652.6	& 3042.7	& -5937.6	& -2491.7	& -748.2	& -867.5 \\ 
  5652.6	& 7623.5	& 1134.7	& -4618.6	& -3824.7	& -334.2	& 1997.5 \\ 
  3042.7	& 1134.7	& 4131.9	& -5329.5	& 16.4	& -611.1	& -3929.4 \\ 
  -5937.6	& -4618.6	& -5329.5	& 9028.2	& 2036.1	& -1430.4	& 3290.3 \\ 
  -2491.7	& -3824.7	& 16.4	& 2036.1	& 2264	& -1088.5	& -1985.8 \\ 
  -748.2	& -334.2	& -611.1	& -1430.4	& -1088.5	& 6713	& 1819.7 \\ 
  -867.5	& 1997.5	& -3929.4	& 3290.3	& -1985.8	& 1819.7	& 5608.4 
\end{pmatrix}  
$$
 
 After eigen-decomposition of $F'$, we get via equation (\ref{eq:phimatrix}) the embedding matrix (after ignoring columns with next to zero eigenvalues)
$$ Y' =
\begin{pmatrix}  
  -71.6	& 9.6	& -14.3	& 0.8 \\ 
  -67.6	& 49.2	& -24.2	& -6.7 \\ 
  -49.2	& -40.3	& 7	& 5.6 \\ 
  90.1	& 17.5	& -24.3	& -2.7 \\ 
  29.8	& -37	& 0.2	& -2.3 \\ 
  -1	& 29.1	& 76.5	& -3 \\ 
  22.2	& 71	& -2.6	& 8.4 
\end{pmatrix}  
$$
 which produces the distance matrix 
$$ D' =
\begin{pmatrix}  
  0	& 41.7	& 58.9	& 162.3	& 112.6	& 116.8	& 113 \\ 
  41.7	& 0	& 97.4	& 160.9	& 132.4	& 122.5	& 96.1 \\ 
  58.9	& 97.4	& 0	& 154.3	& 79.8	& 109.8	& 132.7 \\ 
  162.3	& 160.9	& 154.3	& 0	& 85	& 136.4	& 89.8 \\ 
  112.6	& 132.4	& 79.8	& 85	& 0	& 105.6	& 108.8 \\ 
  116.8	& 122.5	& 109.8	& 136.4	& 105.6	& 0	& 93.2 \\ 
  113	& 96.1	& 132.7	& 89.8	& 108.8	& 93.2	& 0 
\end{pmatrix}  
$$
 
 The sum of squared differences between the corresponding entries in the distance  matrices $D'$ and $D_0$ amounts to  2.090166e-25.

 Not surprisingly, a version of kernel-$k$-means, as described in this paper, was also implemented and produced for the kernel matrix $F'$ the very same clustering. 
 [2, 2, 2, 1, 1, 1, 1].

%% file: kernelFunction_nonEuclideanCase.tex
Consider the following non-Euclidean distance matrix$$ _{nE}D =
\begin{pmatrix}  
  0	& 10	& 20	& 20	& 40	& 40 \\ 
  10	& 0	& 40	& 40	& 20	& 40 \\ 
  20	& 40	& 0	& 40	& 40	& 20 \\ 
  20	& 40	& 40	& 0	& 20	& 10 \\ 
  40	& 20	& 40	& 20	& 0	& 40 \\ 
  40	& 40	& 20	& 10	& 40	& 0 
\end{pmatrix}  
$$
 and the corresponding kernel matrix $$ _{nE}F =
\begin{pmatrix}  
  266.7	& 316.7	& 191.7	& 66.7	& -408.3	& -433.3 \\ 
  316.7	& 466.7	& -308.3	& -433.3	& 291.7	& -333.3 \\ 
  191.7	& -308.3	& 516.7	& -408.3	& -283.3	& 291.7 \\ 
  66.7	& -433.3	& -408.3	& 266.7	& 191.7	& 316.7 \\ 
  -408.3	& 291.7	& -283.3	& 191.7	& 516.7	& -308.3 \\ 
  -433.3	& -333.3	& 291.7	& 316.7	& -308.3	& 466.7 
\end{pmatrix}  
$$
 If we apply kernel-$k$-means clustering with $k=2$, this implies  a clustering 
[ 2, 2, 1, 2, 2, 1] 
with the total value of the cost function  1325 .
Other clusterings would not be better. Check e.g. that the clustering   [1,1,1,2,2,2] produces   the cost function amounting to  1400 which is higher than what kernel-$k$-means produces.

But consider now a different clustering, [1,1,1,2,2,2], 
where you choose weighted cluster centers with weights [10,1,1,10,1,1],
instead of the $k$-means cluster centers. 
Then the cost function will amount to  1175 which is below what kernel-$k$-means produces.

%% file: kernelFunction_Gower7wrong.tex
Gowerr's constant for $_{E}F$ amounts to  $\sigma$=  757.205 . Upon modifying the distance matrix we get the new kernel matrix   $$ _{imp}F =
\begin{pmatrix}  
  582.2	& 253.6	& 128.6	& 3.6	& -471.4	& -496.4 \\ 
  253.6	& 782.2	& -371.4	& -496.4	& 228.6	& -396.4 \\ 
  128.6	& -371.4	& 832.2	& -471.4	& -346.4	& 228.6 \\ 
  3.6	& -496.4	& -471.4	& 582.2	& 128.6	& 253.6 \\ 
  -471.4	& 228.6	& -346.4	& 128.6	& 832.2	& -371.4 \\ 
  -496.4	& -396.4	& 228.6	& 253.6	& -371.4	& 782.2 
\end{pmatrix}  
$$
 which is again non Euclidean, because its lowest eigenvalue is equal      -378.603

%% file: kernelFunction_Gower7correction.tex
Upon modifying the distance matrix according to our Theorem we get the new kernel matrix   $$ _{E}F =
\begin{pmatrix}  
  897.7	& 190.5	& 65.5	& -59.5	& -534.5	& -559.5 \\ 
  190.5	& 1097.7	& -434.5	& -559.5	& 165.5	& -459.5 \\ 
  65.5	& -434.5	& 1147.7	& -534.5	& -409.5	& 165.5 \\ 
  -59.5	& -559.5	& -534.5	& 897.7	& 65.5	& 190.5 \\ 
  -534.5	& 165.5	& -409.5	& 65.5	& 1147.7	& -434.5 \\ 
  -559.5	& -459.5	& 165.5	& 190.5	& -434.5	& 1097.7 
\end{pmatrix}  
$$
 which is now Euclidean, because its lowest eigenvalue is equal      0 
The kernel matrix $_{E}F$ implies  a clustering 
[ 1, 2, 1, 1, 2, 1] 
with the total value of the cost function  4353.821 .
Other clusterings would not do better. Check e.g. that the clustering   [1,1,1,2,2,2] produces   the cost function amounting to  4428.821 which is higher than what kernel-$k$-means produces.

Consider now a different clustering, [1,1,1,2,2,2], 
where you choose weighted cluster centers with weights [10,1,1,10,1,1],
instead of the $k$-means cluster centers. 
Then the cost function will amount to  5907.533 which is again higher than   what kernel-$k$-means produces. In Euclidean space, kernel-$k$-means produces appropriate results.

%% file: kernelFunctionFI.bbl
\begin{thebibliography}{10}

\bibitem{Balaji:2007}
R.~Balaji1 and R.B. Bapat.
\newblock On euclidean distance matrices.
\newblock {\em Linear Algebra and its Applications}, 424(1):108--–117, 2007.

\bibitem{Chitta:2011}
Radha Chitta, Rong Jin, Timothy~C. Havens, and Anil~K. Jain.
\newblock Approximate kernel k-means: Solution to large scale kernel
  clustering.
\newblock In {\em Proceedings of the 17th ACM SIGKDD International Conference
  on Knowledge Discovery and Data Mining}, KDD '11, pages 895--903, New York,
  NY, USA, 2011. ACM.

\bibitem{Dhillon:2004}
I.S. Dhillon, Y.~Guan, and B.~Kulis.
\newblock Kernel k-means: Spectral clustering and normalized cuts.
\newblock In {\em Proceedings of the Tenth ACM SIGKDD International Conference
  on Knowledge Discovery and Data Mining}, KDD '04, pages 551--556, New York,
  NY, USA, 2004. ACM.

\bibitem{Gower:1982}
J.~C. Gower.
\newblock Euclidean distance geometry.
\newblock {\em Math. Scientist}, 7:1--14, 1982.

\bibitem{Gower:1985}
J.~C. Gower.
\newblock Properties of {E}uclidean and non-{E}uclidean distance matrices.
\newblock {\em Linear Algebra and its Applications}, 67:81--97, 1985.

\bibitem{Gower:1986}
J.C. Gower and P.~Legendre.
\newblock Metric and {E}uclidean properties of dissimilarity coefficients.
\newblock {\em Journal of classification}, 3(1):5--48, 1986.
\newblock Here Gower:1982 is cited in theorem 4, but with a different form of
  condditions for D and s.

\bibitem{Handhayania:2015}
T.~Handhayania and L.~Hiryantob.
\newblock Intelligent kernel k-means for clustering gene expression.
\newblock In {\em International Conference on Computer Science and
  Computational Intelligence (ICCSCI 2015) Procedia Computer Science},
  volume~59, pages 171--177, 2015.

\bibitem{MAK:2017}
Mieczyslaw~A. Klopotek.
\newblock On the existence of kernel function for kernel-trick of k-means.
\newblock In {\em Foundations of Intelligent Systems - 23rd International
  Symposium, {ISMIS} 2017, Warsaw, Poland, June 26-29, 2017, Proceedings},
  pages 97--104, 2017.

\bibitem{Li:2013}
C.~Li, M.~Georgiopoulos, and G.~C. Anagnostopoulos.
\newblock Kernel-based distance metric learning in the output space.
\newblock In {\em International Joint Conference on Neural Networks (IJCNN)},
  Dallas, TX, August 04-09, 08/2013 2013.

\bibitem{Nikolentzos:2017}
Giannis Nikolentzos, Polykarpos Meladianos, and Michalis Vazirgiannis.
\newblock Matching node embeddings for graph similarity.
\newblock In {\em Proceedings of the Thirty-First {AAAI} Conference on
  Artificial Intelligence, February 4-9, 2017, San Francisco, California,
  {USA.}}, pages 2429--2435, 2017.

\bibitem{Pavoine:2004}
Sandrine Pavoine, Anne-Béatrice Dufour, and Daniel Chessel.
\newblock From dissimilarities among species to dissimilarities among
  communities: a double principal coordinate analysis.
\newblock {\em Journal of Theoretical Biology}, 228(4):523 -- 537, 2004.

\bibitem{Pevkalska:2002}
Elzbieta Pekalska, Pavel Paclik, and Robert P.~W. Duin.
\newblock A generalized kernel approach to dissimilarity-based classification.
\newblock {\em J. Mach. Learn. Res.}, 2:175--211, March 2002.

\bibitem{Richter:2017}
Ronald Richter, Jan~Eric Kyprianidis, Boris Springborn, and Marc Alexa.
\newblock Constrained modelling of 3-valent meshes using a hyperbolic
  deformation metric.
\newblock {\em Comput. Graph. Forum}, 36(6):62--75, 2017.

\bibitem{Sarma:2013}
T.~Sarma, P~Vishwanath, and B.~Reddy.
\newblock Single pass kernel k -means clustering method.
\newblock {\em Sadhan}, 38, Part 3:407--419, 2013.

\bibitem{Schoenberg:1938}
I.~J. Schoenberg.
\newblock Metric spaces and positive definite functions.
\newblock {\em Trans. Amer. Math. Soc.}, 44:522--536, 1938.

\bibitem{Schoenberg:1935}
I.J. Schoenberg.
\newblock Remarks to {M}aurice {F}réchet’s article “{S}ur la définition
  axiomatique d’une classe d’espace distanciés vectoriellement applicable sur
  l’espace de {H}ilbert”.
\newblock {\em Annals of Mathematics}, 36(3):724--–732, 1935.

\bibitem{Schoelkopf:2001}
Bernhard Sch{\"o}lkopf.
\newblock The kernel trick for distances.
\newblock In {\em Advances in Neural Information Processing Systems 13}, pages
  301--307, Cambridge, MA, USA, April 2001. Max-Planck-Gesellschaft, MIT Press.

\bibitem{Tzortzis:2009}
G.~Tzortzis and Likas A.
\newblock The global kernel k-means algorithm for clustering in feature space.
\newblock {\em IEEE Trans Neural Netw.}, 7(20):1181--94, Jul 2009.

\bibitem{Yaqiang:2018}
Yaqiang Yao and Huanhuan Chen.
\newblock Multiple kernel $k$-means clustering by selecting representative
  kernels.
\newblock \url{https://arxiv.org/abs/1811.00264}, 11 2018.

\end{thebibliography}
